\documentclass{article}
\usepackage{ijcai19}

\usepackage{times}
\usepackage{soul}
\usepackage{url}
\usepackage[hidelinks]{hyperref}
\usepackage[utf8]{inputenc}
\usepackage[small]{caption}
\usepackage{graphicx}
\usepackage{amsmath}
\usepackage{textcomp}
\usepackage{csquotes}
\usepackage{amsfonts}
\usepackage{amssymb}
\usepackage{booktabs}
\usepackage{algorithm}
\usepackage{algorithmic}
\usepackage{thm-restate}
\usepackage{color}
\usepackage{balance}
\usepackage{verbatim}

\newtheorem{proof}{Proof}

\newtheorem{definition}{Definition}
\newtheorem{example}{Ex}

\DeclareMathSymbol{\mh}{\mathord}{operators}{`\-}
\DeclareMathOperator*{\argmax}{argmax}

\title{A Value-based Trust Assessment Model for Multi-agent Systems}

\author{
Kinzang Chhogyal$^1$
\and
Abhaya Nayak$^1$\and
Aditya Ghose$^{2}$ \and
Hoa K. Dam$^{2}$
\affiliations
$^1$Macquarie University, Sydney, Australia\\
$^2$University of Wollongong, Wollongong, Australia
\emails
\{kin.chhogyal, abhaya.nayak\}@mq.edu.au,
\{aditya, hoa\}@uow.edu.au
}

\begin{document}

\maketitle

\begin{abstract}
 An agent's assessment of its trust in another agent is commonly taken to be a measure of the reliability/predictability of the latter's actions. It is based on the trustor's past observations of the behaviour of the trustee and requires no knowledge of the inner-workings of the trustee. However, in situations that are new or unfamiliar, past observations are of little help in assessing trust. In such cases, knowledge about the trustee can help. A particular type of knowledge is that of \emph{values} - things that are important to the trustor and the trustee. In this paper, based on the premise that the more values two agents share, the more they should trust one another,  we propose a simple approach to trust assessment between agents based on values, taking into account if agents trust cautiously or boldly, and if they depend on others in carrying out a task.
\end{abstract}

\section{Introduction}

\noindent Though vastly outnumbered and facing certain defeat  in Thermopylae, Leonidas still trusted that his soldiers would stand and fight for Sparta with their lives. What made him have such faith in them?  It is plausible that his prior experience of sharing the battlefield made him trust them. However, a more compelling reason and one that is of interest to us, could be because they shared common values: they valued their way of life, they valued courage, they valued their freedom and they valued Sparta. 

Autonomous systems such as self-driving cars are becoming a common sight and they have become a source of trepidation in humans. It appears inevitable that we must coexist with them and such fears may be alleviated by designing systems that humans can trust. In computation, there are different perspectives from which to approach trust. An interesting perspective that has largely motivated this work is offered in \cite{roff2018trust} where two dimensions of trust are presented: one that depends on reliability and/or predictability and another that depends \emph{on one’s understanding of other people’s values, preferences, expectations, constraints, and beliefs, where that understanding is associated with predictability but is importantly different from it}. It is this latter dimension which relies on the knowledge of others. 

Many definitions of trust can be found in the literature. We adopt the following definition from  \cite{Lee2004TrustIA}: \emph{the attitude that an agent will help achieve an individual's goals in a situation characterized by uncertainty and vulnerability}. It is important to note that trust arises in situations where i) a trustor expects the trustee to perform some action, and that ii) trustors, in general, have no certainty about the motives and actions of the trustees.  For a survey of trust models, see \cite{sabater2005}.

Out of the `reliability and/or predictability' dimension and the `knowledge dimension',  the focus in AI has largely been on the former. For instance, one of the earliest works in computational trust \cite{marsh1994formalising} was based on this dimension.  The \emph{trustor} in such cases relies on past observations of the \emph{trustee's} behaviour and has no deep knowledge of the trustee. For example, \emph{I trust my car will start in the morning without knowing the inner-workings of the car} \cite{roff2018trust}. The problem with this dimension is that since it relies on past experiences, if situations arise that are either new or unfamiliar, it is not clear how much to  trust or even worse how to trust. This is especially important for autonomous agents as they may find themselves in worlds that are chaotic and ever-changing. They are certain to encounter situations that they have not seen before and choosing how much to trust another agent based on past experiences is futile. This is where trust based on the second dimension can help. The agent's trust in another agent is a function of its knowledge of the latter. Such knowledge could consist of many things but an important factor in the context of trust is knowing what things are important to others, i.e. their values. For instance, if both you and your architect value \emph{beauty}, you can trust your architect to deliver a design that  is beautiful.

This paper is  premised upon why Leonidas trusted his soldiers and why you could trust your architect -- the sharing of common values. It is  reasonable to assume that the more you share values with someone, the more likely you will trust them. We focus on agents that have to rely on other agents to execute certain actions for them but in order to do so they must find the most trustworthy ones. That is, they will seek agents that share their values. We begin by presenting a trust assessment model that relies on both the dimensions -- reliability and value sharing. We then constrain our model to one where only values are used, as that is the focus of this paper. We briefly discuss what values are and how they may be used in trust assessment. Several different ways that trust may be assessed are presented. We end by discussing the limitations of this work and how it may be further extended.

\section{A Trust Assessment Model}

The scenario that we consider in this paper is an environment consisting of autonomous agents that can execute actions.
%
%
%
Our work is motivated by the Belief-Desire-Intention (BDI) agent model \cite{rao1995bdi} but we limit our discussion only to the features of BDI agents that are relevant to our work. Let $\mathbb A = \{ A, B, \hdots \}$ represent the set of all agents. There is also a set $\mathcal A = \{a', a'' \hdots \}$ which represents the set of all possible actions. Note that agents may not be able to execute every action in $\mathcal A$ but they may still be aware of those actions and of other agents that can execute them. The goal of an agent may either be to change the state of the world or get some information about its current state. 
\begin{definition}
Let $A$ be an agent with some goal and $B$ be another agent that can help achieve $A$'s goal by executing action $a'$. We define  $A$'s trust assessment of $B$ w.r.t $a'$ as:
$$
T _{A}(B, a') = \alpha \ T_{A}^{Rel}(B, a') + \beta \ T_{A}^{K} (B, a'),
$$
where $\alpha$ and $\beta$ are weights, $T_{A}^{Rel}(B, a') $ represents $A$'s trust assessment of $B$ based on reliability and predictability, and $T_{A}^{K} (B, a')$ represents $A$'s trust assessment of $B$ based on its knowledge of $B$.
\label{def:trustAssessment}

\end{definition}

\noindent If we take the measure of trust to be the probability with which $A$ thinks $B$ can help achieve its goal by executing $a'$, then $T _{A}(B, a') \in [0,1]$.   Since $T_{A}^{Rel}(B, a')$ relies on past observations, it is implicit that $A$ has a history of executed actions to draw on that involve $B$ and this makes it amenable to machine learning techniques. However, it could turn out that no such history is available; in that case, $T_{A}^{Rel}(B, a')$ can be taken to  be $0$ and therefore, $T _{A}(B, a') = \beta \ T_{A}^{K} (B, a')$. This will be the extent of our discussion of $T_{A}^{Rel}(B, a')$.  We now turn to $T_{A}^{K} (B, a')$ which is the main focus of the paper. The weight $\beta$ is not important and we ignore it in our discussion. In the rest of this paper, we will focus on only one kind of knowledge of the trustee, namely, its values. We refer to  $T_{A}^{K} (B, a')$ as $A$'s \emph{value-based trust assessment} of $B$ w.r.t. action $a'$ or simply \emph{trust assessment} when it is clear from the context.

\subsection{Values}
%


\noindent Values are things that are important to us. According to Schwartz's \emph{Theory of Basic Values} \cite{schwartz2012overview}, all values exhibit six features that include: i) being able to be activated and causing emotions to rise, ii) acting as  goals that can motivate action, iii) guiding the selection of actions and, iv) being able to be ordered by importance. Additionally, in \cite{schwartz2012overview}, ten broad values such as \emph{benevolence, power, security and conformity} are identified under which more concrete values may fall.

Values may also be compatible with each other (\emph{conformity} and \emph{security}) or be in conflict with each other (\emph{benevolence} and \emph{power}).\footnote{Note the same pair of values might conflict in one context and not in another - so they may be context-sensitive. However, we do not take up context-sensitivity in this paper.} Although one could argue that trust (trustworthiness) is itself a value, the central premise of this paper is that trust between two agents arises based on the compatibility of their values.
 This view of trust is in line with \emph{value sensitive design} \cite{friedman2013value} which takes into account human values during the design process of systems which in our case is a trust assessment system. 
 
We assume all agents have values that are explicitly programmed.  The ten broad values mentioned earlier are useful but too coarse for our purpose. Those values are likely to be universal  \cite{schwartz2012overview}, meaning, they are likely present in all agents and differentiating agents based on those values is almost impossible. The values that we consider are therefore taken to be more concrete values which may be classified under these broader values. Agents may share values but they may also have \emph{personal} values unique to them. Agents may have conflicting values but as in  \cite{schwartz2012overview} we take that conflicting values are \emph{not} pursued in a single action. This has an important implication that specific to each action is a set of non-conflicting values that the agent considers important.

Values are assumed to be activated when the state of the world changes due to an agent's own actions or actions of other agents. As in \cite{cranefield2017no}, we assume that for each value of an agent, there is a \emph{value state} that \emph{represents the current level of satisfaction for the value}. Value states could be affected both by an agent's own actions or by the actions of other agents. For instance, an agent that \emph{donates money} would increase the value state of \emph{generosity} for itself.  On the other hand,  if the agent values \emph{the environment}, the value state would decrease for this agent even if it is another agent that \emph{pollutes} the environment. Furthermore, in  \cite{cranefield2017no}, value states are taken to be numbers that do not exceed a certain value. They are also assumed to decay to represent the fact that if no action has been taken in a while that advances an agent's value, its satisfaction decreases. Our concern here is not so much about the actual values but more about the fact that value states can either increase or decrease. Given a set of actions and a set of values, we consider the agent's choice of an action to be guided by the values. More specifically, an agent's choice is such that: i)  it increases the value state of each of its values and/or, ii) it minimises the number of values whose value state is decreased.  The first condition is desirable but is not always achievable. For instance, you respect traffic rules but might run a red light in case there is a person requiring immediate hospitalisation. In this case, the value state for \emph{helpfulness} would increase whereas the value state for \emph{law abidance} would decrease. However, in this paper, we will assume an agent's action increases the value state of each of its values related to that action. This is a strong assumption and will be addressed in the discussion section.

We now formalise the notions that were just discussed. We assume there is a set of all values, $\mathcal V = \{a, b, \hdots\}$, from which an agent's values are drawn. 
We also assume that it is possible for a value $v \in \mathcal{V}$ to have one or more \emph{opposing} (conflicting) values in $\mathcal{V}$. The term ${\sim} v$ is the set of opposing values of $v$. However, if $a \in \mathcal V$ and ${\sim a} = \{b, c\}$, we abuse notation and write ${\sim} a = b = c$ and also let ${\sim} v$ stand for any opposing value of $v$.

%
\begin{definition}
Let $V \subseteq \mathcal{V}$. We say $V$ is \emph{consistent} iff for each $v \in V, \neg \exists v' \in  V \text{ where } v' = \ {\sim} v$. Otherwise, it is \emph{inconsistent}.
\end{definition}
\begin{definition}
Given two sets of values $V$ and $V'$ respectively, the conflict set $V \perp V'$  is defined as  $V \perp V' =  \{v \mid v \in V \text{ and } \exists v' \in V' \text{ where }  v' = \ {\sim} v  \}$.
\label{def:conflictSet}
\end{definition}
\begin{example}
If  $V=\{a\}$, $V'=\{ b, c, d\}$ and $ {\sim} a = b = c $, then $V \perp V' = \{a\}$  and $V' \perp V = \{ b, c\}$.
\label{ex:conflictSet}
\end{example}

\noindent  Ex.\ref{ex:conflictSet} shows $\perp$ is not symmetric. Some basic properties follow from these definitions:

\begin{restatable}{proposition}{propIntersectionConsistent}
Given two sets of values $V, V' \subseteq \mathcal{V}$ if one of $V$  or $V'$ is consistent, then $V \cap V'$ is consistent.  
\end{restatable}
%


\noindent Note that even if $V$ and $V'$ are both inconsistent, $V \cap V'$ could be consistent. For instance, if $V= \{a, b\}$ where $b = {\sim} a$, and $V' = \{a, c, d\}$ where $d = {\sim} c$, then $V \cap V' = \{a\}$ which is consistent. On the other hand, even though both $V$ and $V'$ are consistent, it can be that $V \cup V'$ inconsistent. For instance, if $V=\{a\}$, $V'=\{b\}$, where $b = {\sim} a$, then $V \cup V' = \{a, b\}$  is inconsistent. 

\begin{restatable}{proposition}{propConflictConsistentOne}
Given two sets of values $V, V' \subseteq \mathcal{V}$, if one of $V$  or $V'$ is consistent, then $V \perp V'$ is consistent.
\label{propConflictConsistentOne}
\end{restatable}

\begin{restatable}{proposition}{propConflictConsistentTwo}
Given two sets of values $V, V' \subseteq \mathcal{V}$, $V \perp V'$ is inconsistent iff both $V$ and $V'$ are individually inconsistent and there is some value $v$ such that both $v, {\sim} v$ in $V$ and $V'$.
\label{propConflictConsistentTwo}
\end{restatable}
%


\begin{restatable}{proposition}{propDistributivityOverCapCup}
Given three sets of values $V, V', V'' \subseteq \mathcal{V}$:
\begin{enumerate}
\item $(V \cap V') \perp V''  = (V \perp V'') \cap (V' \perp V'')$,
\item $(V \cup V') \perp V''  = (V \perp V'') \cup (V' \perp V'')$.
\end{enumerate}
\label{propDistributivityOverCapCup}
\end{restatable}
%


\noindent Proposition \ref{propDistributivityOverCapCup} shows that $\perp$  distributes over $\cap$ and $\cup$. However, the converse doesn't hold, i.e.,   $\cap$ and $\cup$ do not distribute over $\perp$.  We show them below along with the non-associativity of $\perp$ for the sake of completeness. For the counterexamples below, let $V=\{a\}$, $V' = \{b\}$ , and $V'' = \{a\}$ where $b = {\sim} a$.

\begin{enumerate}
\item $(V \perp V'') \cup V' \not = (V \cup V') \perp (V'' \cup V')$: \\ Ex. We get $(V \perp V'') \cup V' = \{\} \cup \{b\} = \{b\}$, and $ (V \cup V') \perp (V'' \cup V') = \{a, b\} \perp \{a, b\} = \{a, b\}$.
\item $(V \perp V') \cap V'' \not = (V \cap V'') \perp (V' \cap V'')$: \\ Ex. We get $(V \perp V') \cap V'' = \{a\} \cap \{a\} = \{a\}$, and $ (V \cap V'') \perp (V' \cap V'') = \{a\} \perp \{\} = \{\}$. 
\item $ (V \perp V') \perp V'' \not = V \perp (V' \perp V'')$. \\ Ex. We get $ (V \perp V') \perp V'' = \{a\} \perp \{a\} = \{\}$, and $V \perp (V' \perp V'') = \{a\} \perp \{b\} = \{a\}$.
\end{enumerate}

\subsection{Value-based Trust Assessment}
\begin{definition}
An agent $A$'s \emph{value set}, $\mathcal{V}_{A}$, is a subset of $\mathcal{V}$.
\label{def:value set-agent}
\end{definition}

\begin{definition}
Given an agent $A$ and an action $a' \in \mathcal{A}$, the \emph{action value set} associated with $a'$, denoted as $V_A^{a'}$, is a subset of $\mathcal{V}_A$ that is consistent. 
\label{def:action-value set-agent}
\end{definition} 

\noindent When it is clear from the context, we write $V_A^{a'}$ simply as $V_A$. Def. \ref{def:action-value set-agent} follows from what we mentioned earlier that conflicting values cannot be pursued in a single action. We don't specify how $V_A$ is formed but the values in it should consist of values that are important w.r.t $a'$.  For example, if I am about to \emph{buy} a new piece of furniture, I might care about \emph{functionality} and not \emph{beauty}; so \emph{functionality} would be in $V_A$. Note that we did not mention whether $a'$ can be executed by $A$ or not. $A$ might not be able to execute an action but it can still be aware of the action and the values that are important relative to it. For instance, you may not know how to drive but in asking someone to drive, you would still value \emph{safety} and \emph{comfort}. The action value set could also consist of \emph{core} values that are important to the agent regardless of any action. As mentioned earlier, if $A$ can execute $a'$, it is assumed that all values in $V_A$ increase their value state after executing $a$.

 
%
%

%
\subsection*{Basic Trust Assessment}
%

The first case we consider is how an agent might assess its trust in another agent when requesting a particular action to be executed.

\begin{definition}[Two Agent - Independent]
Given an action $a'$, two agents $A$ and $B$ with value sets $V_A$ and $V_B$, the value-based trust assessment $Tr^K_A(B, a')$ of $B$ by $A$ is defined as:
\begin{equation*}
\begin{aligned}
Tr^{K}_A(B, a') = \ \mid V_A \cap V_B \mid - \mid V_A \perp V_B \mid
\end{aligned}
\end{equation*}
\vspace{-5mm}
\label{def:trustAssessment2agents}
\end{definition}

\noindent Intuitively, the level of trust $A$ places in $B$ is determined both by the values they share, $ \mid V_A \cap V_B \mid$, and the extent to which $A$'s values conflict with $B$'s,  $\mid V_A \perp V_B \mid$. Note that $V_B =  V_B^{a'}$. Also,  $V_A \perp V_B$ is consistent from Proposition \ref{propConflictConsistentOne}.
We will at times annotate $Tr^{K}_A(B, a')$ and write it as $Tr^{K}_A(B, a')\allowbreak[independent]$ since $A$ is not acting on behalf of any agent. This is mainly to make the presentation simpler when comparing different trust assessment functions. The following properties result directly from Def. \ref{def:trustAssessment2agents} :
\begin{enumerate}
\item   if $V_A \perp V_B = \{\}$, $Tr^K_A(B, a') \geq 0$, 
\item if $V_A \cap V_B$ = \{\}, $Tr^K_A(B, a') \leq 0$, and 
\item  if $V_A \cap V_B = \{\}$ and $V_A \perp V_B = \{\}$, $Tr^K_A(B, a') = 0$.
\end{enumerate}
%

%
\begin{example}
Let $V_A= \{a, b, c, d\}$ and $V_B = \{a, b, e, f, g\}$, where ${\sim} c = e = f$ and $a'$ be some action. We get $Tr^K_A(B, a') = \ \mid \{a, b \} \mid - \mid  \{c\}  \mid \ = 2 - 1 = 1$.
\end{example}

\noindent Next, we consider the case where three agents are involved. Say $A$ asks $B$ to \emph{build} her a red chair. However, $B$ is only a carpenter and not a painter. So, $B$ must also request a trustworthy painter to \emph{paint} the chair. We have to be careful here as there are two value sets concerning $B$: $V_B^{build}$ and $V_B^{paint}$. The question is which value set does $B$ use in order to pick a painter $C$? Since $B$ is fulfilling $A$'s request, we assume that $V_B^{build}$ supersedes $V_B^{paint}$ and is the value set used to choose $C$,  i.e. $V_B = V_B^{build}$.  If $B$ were acting independently of $A$, then it would be more appropriate to take $V_B$ as $V_B^{paint}$. We propose two ways that $B$ might adopt to choose $C$.
\begin{definition}[Three Agents - Cautious]
Given actions $a'$ and $a''$, three agents $A$, $B$ and $C$ where $B$ is executing $a'$ on behalf of $A$ and $C$ is executing $a''$ on behalf of $B$, and value sets $V_A = V_A^{a'}$, $V_B =  V_B^{a'} $ and, $V_C =  V_C^{a''}$, the cautious trust assessment of $C$ by $B$ is defined as:
\begin{equation*}
\begin{aligned}
Tr^K_B(C, a'') = \mid (V_A \cap V_B) \cap V_C \mid - \mid (V_A \cup V_B) \perp V_C \mid
\end{aligned}
\end{equation*}
\vspace{-5mm}
\label{def:trustAssessment3agentsCautious}
\end{definition}

\noindent Here, we say $B$ trusts cautiously. It tries to pick an agent that has the most values common to both itself and $A$. On the other hand, it avoids agents that have a lot of values in conflict with itself or $A$. At times we use the annotated form $Tr^K_B(C, a'')\allowbreak[cautious]$. Note that the relevant action in $Tr^K_B(C, a'')$  is $a''$ though $V_B$ is defined relative to $a'$, i.e. $V_B^{a'}$. $ (V_A \cup V_B)$ may be inconsistent but since $ V_C$ is consistent, from Proposition  \ref{propConflictConsistentOne}, we know $(V_A \cup V_B) \perp V_C$ is consistent.

\begin{example}
As in the previous example, let $V_A= \{a, b, c, d\}$ and $V_B = \{a, b, e, f, g\}$, where ${\sim} c =  e  = f $. Let $V_C=\{a, e, h\}$ where ${\sim} g = h$. $Tr^K_B(C, a'') = \ \mid \{a, b\} \cap \{a, e, h\} \mid - \mid \{a, b, c, d, e, f, g\} \perp \{a, e, h\}\mid \ =  \mid \{a\} \mid  - \mid \{ c, g \}\mid \ = 1 - 2 = -1$.
\label{example:cautious}
\end{example}
\begin{definition}[Three Agents - Bold]
Given actions $a'$ and $a''$, three agents $A$, $B$ and $C$ where $B$ is executing $a'$ on behalf of $A$ and $C$ is executing $a''$ on behalf of $B$,  and value sets $V_A = V_A^{a'}$, $V_B =  V_B^{a'} $ and, $V_C =  V_C^{a''}$, the bold trust assessment of $C$ by $B$ is defined as:
\begin{equation*}
\begin{aligned}
Tr^K_B(C, a'') = \mid (V_A \cup V_B) \cap V_C \mid - \mid (V_A \cup V_B) \perp V_C \mid
\end{aligned}
\end{equation*}
\vspace{-5mm}
\label{def:trustAssessment3agentsBold}
\end{definition}

\noindent Here, we say $B$ trusts boldly. The annotated form is $Tr^K_B(C, a'')\allowbreak[bold]$. As in the previous case, values common to all three agents are considered but so are values that $A$ and $B$ independently share with  $C$ for  assessing the trust in $C$. In general, $B$ places at least as much trust in agents as it would have when being cautious as shown in Proposition \ref{propBoldGreaterThanCautious} below.

\begin{example}
As before, $V_A= \{a, b, c, d\}$ and $V_B = \{a, b, e, f, g\}$, where ${\sim} c = e = f$. Let $V_C=\{a, e, h\}$ where ${\sim} g = h$. $Tr^K_B(C, a') = \mid \{a, b, c, d, e, f, g\} \cap \{a, e, h\} \mid - \mid \{a, b, c, d, e, f, g\} \perp \{a, e, h\}\mid =  \mid \{a, e\} \mid  - \mid \{ c, g \}\mid = 2 -2 = 0$.
\label{example:bold}
\end{example}
\begin{restatable}{proposition}{propBoldGreaterThanCautious}
Given actions $a'$ and $a''$, three agents $A$, $B$ and $C$ with value sets $V_A$, $V_B$ and $V_C$ where $B$ is executing $a'$ on behalf of $A$ and $C$ is executing $a''$ on behalf of $B$,  $Tr^K_B(C, a'')\allowbreak[bold] \geq Tr^K_B(C, a'')\allowbreak[cautious]$. 
\label{propBoldGreaterThanCautious}
\end{restatable}

\noindent When $B$ trusts boldly or cautiously, it assesses its trust in $C$ for executing $a''$ with $A$'s value set $V_A$ in mind. It is interesting to see what $B$'s trust in $C$ would be if it ignores $V_A$. We say $B$ is acting semi-independently because we still take $V_B$ as $V_B^{a'}$ and not $V_B^{a''}$. The definition for $Tr^K_B(C, a'')\allowbreak[semi\mh independent]$  is the same as in Def. \ref{def:trustAssessment2agents}:

\begin{definition}[Three Agents - Semi-Independent]
Given actions $a'$ and $a''$, three agents $A$, $B$ and $C$ with value sets $V_A = V_A^{a'}$, $V_B =  V_B^{a'} $ and, $V_C =  V_C^{a''}$, the trust assessment of $C$ by $B$ is defined as $Tr^K_B(C, a'')[semi\mh independent] = \ \mid V_B \cap V_C \mid - \mid V_B \perp V_C \mid$.
\label{def:trustAssessment3agentsSemiIndependent}
\end{definition}
\begin{example}
As before,  $V_B = \{a, b, e, f, g\}$, where ${\sim} c = e = f$ and $V_C=\{a, e, h\}$ where ${\sim} g = h$. $Tr^K_B(C, a'')\allowbreak[semi\mh independent] = \ \mid \{a, e\} \mid - \mid \{g\}  \mid \ = 2 - 1 = 1$.
\label{example:bAsInitiator}
\end{example}

\noindent From Ex.\ref{example:cautious}, Ex.\ref{example:bold} and Ex.\ref{example:bAsInitiator}, we see that $Tr^K_B(C, a'')\allowbreak[semi\mh independent] $ is greater than $Tr^K_B(C, a'')\allowbreak[cautious]$ or $Tr^K_B(C, a'')\allowbreak[bold]$ . In other words, trust that $B$ places in $C$ when acting semi-independently is greater than when it is acting on behalf of $A$. However, this only holds in general  between $Tr^K_B(C, a'')\allowbreak[semi\mh independent] $ and $Tr^K_B(C, a'')\allowbreak[cautious] $, and is shown in the next proposition.

\begin{restatable}{proposition}{propSemiIndependentGreaterThanCautious}
Given actions $a'$ and $a''$, three agents $A$, $B$ and $C$ with value sets $V_A = V_A^{a'}$, $V_B =  V_B^{a'} $ and, $V_C =  V_C^{a''}$, $Tr^K_B(C, a'')\allowbreak[semi\mh independent] \geq Tr^K_B(C, a'')\allowbreak[cautious] $.
\label{propSemiIndependentGreaterThanCautious}
\end{restatable}

\noindent The following counterexample shows that $Tr^K_B(C, a'')\allowbreak[semi\mh independent] \ \geq \ Tr^K_B(C, a'')\allowbreak[bold]$ is not true in general.
\begin{example}
As before, let $V_A=\{a, b, c, d\}$ and $V_B = \{a, b, e, f, g\}$, where ${\sim} c = e = f$. We change $V_C$ to $\{d, h\}$ where ${\sim} g = h$. $Tr^K_B(C, a'')\allowbreak[semi\mh independent] = \ \mid \{\} \mid - \mid \{g\} \mid \ = 0  - 1 = -1$. $Tr^K_B(C, a'')\allowbreak[bold] = \ \mid \{a, b, c, d, e, f, g \} \cap \{d, h \}\mid - \mid \{a, b, c, d, e, f, g \} \perp \{d, h \} \mid \ =  \ \mid \{d\} \mid - \mid \{g\} \mid = 1 - 1 = 0$.
\end{example}

\noindent For the special case, when no two of $V_A$, $V_B$, $V_C$ have conflicting values with each other, we have the following result:

\begin{restatable}{proposition}{propCautiosLeqIndependent}
Given actions $a'$ and $a''$, three agents $A$, $B$ and $C$ with value sets $V_A$, $V_B$ and $V_C$ that have no conflicting values with each other, $Tr^K_B(C, a'')\allowbreak[cautious]  \leq Tr^K_B(C, a'')\allowbreak[semi\mh independent] \leq Tr^K_B(C, a'')\allowbreak[bold] $.
\label{propCautiosLeqIndependent}
\end{restatable}
\subsection*{Trust Sequences}
%

We now turn our attention to trust sequences when a series of agents are involved in assessing trust.

\begin{example}
Consider agent $A$ has to achieve a goal that requires the execution of a particular action $a'$. $A$, however, cannot execute $a'$ and instead must rely on another agent. Assume $A$ is only aware of agents $B$ and $C$ that can execute $a'$. 
\end{example}

\noindent In order to pick the \emph{best} one amongst the two, $A$ chooses the one that it believes to be more trustworthy. It does this by assessing its trust in $B$ and $C$, $T^K_{A}(B, a')[independent]$ and $T^K_{A}(C, a')[independent]$ respectively. 

\begin{example}(cont.)
Suppose $A$ has picked $B$ to execute the action as $T^K_{A}(B, a')[independent] > T^K_{A}(C, a')[independent]$.
\label{ex:example1}
\end{example}

\noindent As seen in the example, $A$ uses a simple rule to pick $B$ or $C$. There are two reasons for this: i) $A$ can maximise the chance of its value states increasing, by picking an agent with whom it shares the most number of values, and ii) by choosing agents with whom it has fewer conflicting values, it minimises the chance of  its values being violated. The best scenario for $A$ is the case where either $V_A \subseteq V_{B} \text{ or } V_A \subseteq V_C$. 

\begin{example}(cont.)
Assume that $B$, in turn, has to request either $D$ or $E$ to execute another action $a''$ to fulfil $A$'s request.
\label{ex:example1}
\end{example}

\noindent Similar to what $A$ did, $B$ assesses its trust in $D$ and $E$. Since three agents will be involved $A$, $B$ and, $D$ or $E$, we use either Def.  \ref{def:trustAssessment3agentsCautious} or Def. \ref{def:trustAssessment3agentsBold}. Similar to the case for two agents, $B$ picks the greater of $T^K_{B}(D, a'')$ and $T^K_{B}(E, a'')$. 

\begin{example}(cont.)
Assume $B$ chooses $D$ using Def. \ref{def:trustAssessment3agentsCautious} who then executes $a''$ which is the last action to be executed. The trust assessments between $A$, $B$ and $D$, where $B$ and $D$ are the chosen agents form a trust assessment sequence as shown:

\begin{equation*}
A \xrightarrow[]{Tr^K_{A}(B, a')} 
B \xrightarrow[]{Tr^K_{B}(D, a'')}
D
   \label{eqn:trustSequenceExample}
\end{equation*}
\end{example}

\noindent We now formally define a trust assessment sequence.

\begin{definition}
A value-based trust sequence or simply a trust sequence is a sequence of trust assessments, $T^K_{A_i}(A_{i+1}, a_i)$, where $1 \leq i < n$,   $T^K_{A_i}(A_{i+1}, a_i)$ represents agent $A_i$'s trust assessment of agent $A_{i+1}$ w.r.t to action $a_i$ and  $A_i \not = A_{i+1}$. 
\label{def:trustSequence}
\end{definition}
\noindent Shown below is a way to visualise a trust sequence. Trust assessments on either side are surrounded by the agents involved.
\begin{equation*}
\begin{aligned}
\resizebox{0.99\hsize}{!}{$
A_1 \xrightarrow {Tr^K_{A_1}(A_2, a_1)} 
A_2 \xrightarrow{Tr^K_{A_2}(A_3, a_2)}
\hdots
A_{n-1} \xrightarrow{Tr^K_{A_{n-1}}(A_{n}, a_{n-1})}
A_{n}
$}
\end{aligned}
\end{equation*}

\noindent The trust sequence above is initiated by $A_1$ (\emph{initiator}) and $Tr^K_{A_1}(A_2, a_1)$ is the \emph{initial assessment}. All other assessments will be referred to as \emph{subsequent assessments}. The last agent in the sequence to execute an action is $A_{n}$ and is called the \emph{terminator}. For all $i >1$, each $A_i$ represents the agent that was chosen to execute action $a_{i-1}$ by agent $A_{i-1}$.  The \emph{length} of the sequence is equal to the number of trust assessments, i.e. $n-1$ above. The condition $A_i \not =  A_{i+1}$ prevents sequences where agents assess trust in themselves.\footnote{It may be possible that an agent appears again in some other place in the sequence.} The number of agents involved in the sequence is therefore \emph{at most} $n$. In this paper, we only consider sequences where at each step, an agent only has one trustee. For instance, in Ex.\ref{eqn:trustSequenceExample}, there are no other agents besides $B$ that $A$ asks to execute an action and similarly there is only $D$ for $B$. This leads a sequence that has no branches. Ex.\ref{eqn:trustSequenceExample} already showed how trust sequences are generated and now we present it more formally.

\begin{definition}
Let $i \geq 1 $, $A_i \in \mathbb{A}$ be an agent looking for another agent to execute action $a_i$. The value set of $A_i$ is $V_{A_i}$.  For each $X \in \mathbb{A}$ where $X \not = A_i$, that can help execute $a_i$, we define:
\begin{equation*}
\begin{aligned}
\centering
A_{i+1} = \argmax_X \ Tr^K_{A_i}(X, a_i),
\end{aligned}
\end{equation*} 
where if $i = 1$, $Tr^K_{A_1}(X, a_1)$ is given by Def. \ref{def:trustAssessment2agents} and if $i>1$, $Tr^K_{A_i}(X, a_i)$ is given by one of Def. \ref{def:trustAssessment3agentsCautious} or Def. \ref{def:trustAssessment3agentsBold}.
\label{def:agentSelection}
\end{definition}

\noindent It is clear all trust sequences use Def. \ref{def:trustAssessment2agents}  but differ on whether they use Def. \ref{def:trustAssessment3agentsCautious} or Def. \ref{def:trustAssessment3agentsBold}. This point forward by a \emph{cautious} trust sequence we mean one that uses Def. \ref{def:trustAssessment3agentsCautious} and by a \emph{bold} trust sequence we mean one that use Def. \ref{def:trustAssessment3agentsBold} for all $i > 1$.

\begin{definition}
Given a trust sequence $\mathcal{S}$ of length $n-1$, the aggregate trust of the trust sequence is equal to $\sum\limits_{i=1}^{n-1} Tr_{A_i}^K(A_{i+1}, a_i)$ and is denoted as $Q(\mathcal{S})$.
\label{def:aggregateTrustSequence}
\end{definition}

\noindent During each trust assessment step in the sequence, we are computing the difference between the number of  values that are shared and the number of values that are in conflict; $Q(S)$ is simply the sum of those differences.  If it is positive, then as a whole there are more values \emph{preserved} between each step of the sequence compared to the number of values that are in conflict; if it is negative, the converse is true. Def. \ref{def:aggregateTrustSequence} also allows us to compute the aggregate trust of a \emph{subsequence}: $\sum_{i}^{j} Tr_{A_i}^K(A_{i+1}, a_i)$, where $1 \leq i \leq j$ and $ j \leq n-1$. \\

\noindent In Def. \ref{def:agentSelection}, $A_i$ may trust either boldly or cautiously to choose an agent $A_{i+1}$. An interesting question to ask is whether $A_i$ being bold or cautious makes any difference at all, i.e. will $A_i$ always select the same agent irrespective of whether it is trusting boldly or cautiously? As the example below shows, being cautious or bold matters.

\begin{example}
Given actions $a'$ and $a''$ and four agents $A$, $B$, $C$ and $D$ where $B$ is executing $a'$ on behalf of $A$ and has to choose one between $C$ and $D$ for executing $a''$, let $V_A = \{a, b, c, e\}$, $V_B=\{a, b\}$, $V_C = \{b\}$ and $V_D=\{c, e\}$. 
Consider $Tr^K_B(\cdot)\allowbreak[cautious]$ first:  $Tr^K_B(C, a'')\allowbreak[cautious] \ = \ \mid (V_A \cap V_B) \cap V_C \mid - \mid  (V_A \cup V_B) \perp V_C  \mid \ = \ \mid  \{a, b\} \cap \{b\} \mid - \mid  \{a, b, c, e\} \perp \{b\} \mid = \mid \{b\} \mid - \mid \{\}\mid \ = \  1 - 0 = 1$. Similarly, $Tr^K_B(D, a'')\allowbreak[cautious] \ = \ \mid  \{a, b\} \cap \{c, e\} \mid - \mid  \{a, b, c, e\} \perp \{c, e\} \mid \ = \ \mid \{\} \mid - \mid \{\}\mid \ = \ 0 - 0 = 0$.  $B$ will choose $C$ if trusting cautiously.
Consider $Tr^K_B(\cdot)\allowbreak[bold]$ now:
$Tr^K_B(C, a'')\allowbreak[bold] \ = \ \mid (V_A \cup V_B) \cap V_C \mid - \mid  (V_A \cup V_B) \perp V_C \mid \  = \ \mid \{a, b, c, e\} \cap \{b\} \mid - \mid \{a, b, c, e\} \perp \{b\} \mid \ = \ \mid \{b\}\mid - \mid \{\} \mid \ = \ 1 - 0  = 1$.
$Tr^K_B(D, a'')\allowbreak[bold] \ = \ \mid \{a, b, c, e\} \cap \{c, e\} \mid - \mid \{a, b, c, e\} \perp \{c, e\} \mid \ = \ \mid \{c, e\}\mid - \mid \{\} \mid = 2 - 0  = 2$.  So, $B$ will choose $D$ if trusting boldly which if different from the previous case.
\end{example}

\noindent Intuitively, we think of $Tr^K_{A_1}(A_2, a_1)$  as representing $A_1$'s trust assessment of $A_2$ w.r.t $a_1$. What is not clear is whether $Tr^K_{A_1}(A_2, a_1)$ should be updated to  $Q(\mathcal{S})$? The reason for this is because $A_1$'s trust in $A_2$ also depends on whether $A_2$ has chosen a trustworthy agent $A_3$ that can help fulfil $A_1$'s goal. Assuming we do so, the implication of Theorem \ref{theoremBoldGeqCautios} below is that if $Q(\cdot)$ is used to update $A$'s trust in $B$, then the updated value of $A$'s trust in $B$ will be greater if agents in the sequence trust boldly and not cautiously.

\begin{restatable}{theorem}{theoremBoldGeqCautios}
The aggregate trust of the trust sequence $\mathcal{S'}$ resulting from  $Tr^K_{A_i}(A_{i+1}, a_i)\allowbreak[bold]$ is greater than or equal to  the aggregate trust of the trust sequence $\mathcal{S}$ resulting from $Tr^K_{A_i}(A_{i+1}, a_i)\allowbreak[cautious]$, i.e. $Q(\mathcal{S'}) \geq Q(\mathcal{S})  $.
\label{theoremBoldGeqCautios}
\end{restatable}
\section{Discussion}
%
We discuss some limitations of our work and how it may be expanded on in the future. \\

\noindent \textbf{\underline{Bias in bold agents} }
Consider again Def. \ref{def:trustAssessment3agentsBold} of a bold agent:
$$
Tr^K_B(C, a'') = \mid (V_A \cup V_B) \cap V_C \mid - \mid (V_A \cup V_B) \perp V_C \mid
$$
Say $B$ has selected $C$ as $Tr^K_B(C, a'')$ is the maximum. For simplicity, assume there are no conflicting values in $V_A$ , $V_B$ and $V_C$. We know $ (V_A \cup V_B) \cap V_C =  (V_A \cap V_C) \cup (V_B \cap V_C)$. Assume that $\mid V_B \cap V_C \mid $ is much bigger than $\mid V_A \cap V_C \mid $. Observe that $C$ is largely biased towards $B$ compared to $A$ as they share more values. This means in future trust assessments starting with $C$, $A$'s values could be ignored as more of $B$'s values carry over to the next step in the sequence compared to $A$'s. Now if there happened to be another agent $D$ such that $Tr^K_B(D, a'') $ is only slightly smaller than $Tr^K_B(C, a'') $ but $\mid V_B \cap V_ D\mid $ is only slightly bigger than $\mid V_A \cap V_D \mid $, it seems $D$ might be a better choice than $C$ because as many of  $A$'s values are as likely to be preserved as $B$'s. This leads to the slightly more complex definition for bold agents below: 

\vspace{-3mm}

\begin{equation*}
\begin{aligned}
Tr^K_B(C, a'') & = \  \mid (V_A \cup V_B) \cap V_C \mid -  \mid (V_A \cup V_B) \perp V_C \mid \\  
 & \ \ \ - \  abs(\mid V_A \cap V_C \mid - \mid V_B \cap V_C \mid )
\end{aligned}
\end{equation*}

\noindent A similar kind of bias might exist in the subtrahend $\mid (V_A \cup V_B) \perp V_C \mid$ of Def. \ref{def:trustAssessment3agentsBold}, i.e. between $V_A \perp V_C$ and $V_B \perp V_C$. However, we think minimising the total number of conflicting values heavily outweighs the importance of minimising the bias in this case, so accounting for it is probably unnecessary. \\
\noindent \textbf{\underline{Aggregate trust of a sequence and trust update}}
We mentioned previously the possibility of updating $Tr^K_{A_1}(A_2, a_1)$ to $Q(\mathcal{S})$ or some other value that is a function of it. The case where $Q(\mathcal{S}) < Tr^K_{A_1}(A_2, a_1)$ seems plausible as we can reason that $A_1$ may have overestimated its trust in $A_2$ because it had no knowledge of other agents involved. However, if $Q(\mathcal{S}) > Tr^K_{A_1}(A_2, a_1)$, explaining why $A_1$'s trust in $A_2$ should  increase is not easy. This suggests that $Q(\mathcal{S})$ as a basis of trust update might have to be applied in a more sophisticated way. \\

\noindent \textbf{\underline{Value Preservation}}
Given a trust sequence $\mathcal{S}$ of length $n$, it would be convenient to have a measure which at a minimum could tell us whether a value in the initiator $A_1$ is also in the terminator $A_n$ without having to inspect the values of all agents involved. The aggregate trust of the sequence, $Q(\mathcal{S})$, doesn't seem to have the right characteristics for this. A multiplicative measure based on the ratio between the number of values preserved and the number of values in conflict for each trust assessment is one possible option to explore. \\
 
\noindent \textbf{\underline{Value Preferences}}
We did not consider preferences over values such as in \cite{Serramia:2018:MVN:3237383.3237891}. Suppose you have to choose between two hotels, one in the Downtown area close to all the local attractions and the other cheaper but requiring more travel. If you value \emph{convenience} more than \emph{price}, then you would choose the Downtown hotel whereas if you value \emph{price} more, you would book the cheaper one. When another agent is involved, you will likely choose an agent that has preferences over values similar to yours. This requires more knowledge and also brings additional complexity. A possible way of doing this is to modify the trust assessment functions  in Def. \ref{def:trustAssessment2agents}, Def. \ref{def:trustAssessment3agentsCautious} and Def. \ref{def:trustAssessment3agentsBold} so that they use a measure such as Kendall's tau distance \cite{kendall1938new}.
  \\
\noindent \textbf{\underline{Value States}}
 Although we mentioned that values can be activated and their value states can either increase or decrease, we did not  consider it in our model. Incorporating this information into will be an interesting way to build on the model. We briefly discuss one way this might be done. Let  $A$ and $B$ be two agents with value sets $V_A$ and $V_B$ and $a'$ be an action that $B$ is executing on $A$'s behalf. Let $V_B \uparrow$ and $V_B \downarrow$ be the set of values in $V_B$ whose value state increases and decreases due to the execution of $a'$ respectively. Then: 
\begin{equation*}
\begin{aligned}
V_A \uparrow V_B  & =  \{v \ | \ v \in (V_A \cap V_B) \cap V_B\uparrow \}  \text{ and } \\
V_A \downarrow V_B  & =  \{v \ | \ v \in (V_A \cap V_B) \cap V_B\downarrow \}.
\end{aligned}
\end{equation*}

\noindent $V_A \uparrow V_B$ are values shared by $A$ and $B$ whose value state increases and, $V_A \downarrow V_B$ are values share by $A$ and $B$ whose value state decreases.  We could then rewrite the trust assessment function in Def. \ref{def:trustAssessment2agents} for two agents as: 
\vspace{1mm}
\begin{equation*}
\resizebox{0.99\hsize}{!}{$
Tr^{K}_A(B, a')  = \ \alpha \mid V_A \uparrow V_B \mid -  \ \beta \mid V_A \downarrow V_B \mid - \ \gamma \mid V_A \perp V_B \mid,
$}
\end{equation*}
\vspace{-2mm}
\hspace{0mm} \\
\noindent where $\alpha, \beta$ and $\gamma$ are weighting factors. Note that $\mid V_A \cap V_B \mid$ in Def. \ref{def:trustAssessment2agents} has been replaced by $\alpha \mid V_A \uparrow V_B \mid -  \ \beta \mid V_A \downarrow V_B \mid$. Both $V_A \uparrow V_B, V_A \downarrow V_B \subseteq V_A \cap V_B$ and  in Def. \ref{def:trustAssessment2agents} they both contribute positively. We subtract them but we want to be careful that they don't equal to zero if  $\mid V_A \uparrow V_B \mid  =  \mid V_A \downarrow V_B \mid $ and thus the use of weighting factors. Values in $V_A \perp V_B$ could also increase and decrease but since they are all in conflict with $A$, we do not differentiate between such values.  Similar functions for both  Def. \ref{def:trustAssessment3agentsCautious} and Def. \ref{def:trustAssessment3agentsBold} can be constructed. \\
 %
 %
 
 %
%

 %
\noindent \textbf{\underline{Public Values and Action Decomposition}}
We assumed that when agent $A$ is assessing its trust in agent $B$, the values of $B$ are publicly visible to  $A$, i.e. $A$ is \emph{certain} of $B$'s values. This is quite a strong assumption. A way to circumvent this is to instead consider the set of values that $A$ \emph{believes} $B$ has. Also, in an earlier example, we considered the task to \emph{build a red chair} and we alluded to the fact that there were two actions involved: \emph{build} and \emph{paint}. More work is required on this aspect of decomposing complex actions into simpler ones. \footnote{We are thankful to a anonymous for pointing these issues out and for suggesting that  instead of knowing for certain, agents could perhaps hold beliefs of what another agent's values are.}

\section{Conclusion}
We presented a simple approach to how values can be used by agents to assess their trust in each other. We defined the notion of value-based trust assessment functions and showed how they lead to trust sequences. Many of the ideas in this paper could be further expanded upon and explored in more detail, and there is much to uncover about how values and trust are related.
We leave it to our future research.

\section*{Acknowledgement}
%
We would also like to thank anonymous referees for their comments.



\clearpage
\bibliographystyle{named}
\balance
\bibliography{ijcai19}



\clearpage
\section{Appendix}

\nobalance

\propIntersectionConsistent*
\begin{proof}
1. Follows from the fact that to make  $V \cap V'$ inconsistent it must be that there is a $v$ such that  both $v$ and ${\sim} v$ are in $V$ and $V'$. But at least one is consistent, so it can't be that $V \cap V'$ inconsistent.
\end{proof}
\propConflictConsistentOne*
\begin{proof}
There are three cases to consider. (1) $V$ consistent, $V'$ consistent: Assume $V \perp V'$ is inconsistent. This means there is a value $v$ such that both $v, {\sim} v \in V \perp V'$ and from Def. \ref{def:conflictSet} this implies that both $v, {\sim} v \in V$. However, $V$ is consistent and we get a contradiction. (2)  $V$ consistent, $V'$ inconsistent: Proof similar to case 1. (3) $V$ inconsistent, $V'$ consistent: Assume $V \perp V'$ is inconsistent. Again for some $v$, both  $v, {\sim} v \in V \perp V'$. From Def \ref{def:conflictSet}, it must be that $v \in V$ and ${\sim} v \in V'$, and ${\sim} v \in V$ and $v \in V'$. However, as $V'$ is consistent, it cannot have both $v$ and $v'$ which gives us a contradiction. $\Box$ 
\end{proof}
\propConflictConsistentTwo*
\begin{proof}
\underline{Left to Right}: Assume  $V \perp V'$ is inconsistent. Then from Prop. \ref{propConflictConsistentOne}, both must $V$ and $V'$ are inconsistent. Assume for contradiction, no value $v$ such that both $v, {\sim} v$ in $V$ and $V'$. By Def \ref{def:conflictSet}, for any $v', {\sim} v' \in V$, at most one of ${\sim} v'$ or $v'$  in $V \perp V'$ as $V'$ cannot contain both $v$ and $v'$. This means $V \perp V'$ obtained is consistent and leads to a contradiction. \underline{Right to Left}: Assume $v, {\sim} v$ in $V$ and $V'$. By Def \ref{def:conflictSet}, both  $v, {\sim} v$ in $V \perp V'$ which makes it inconsistent. $\Box$
\end{proof}
\propDistributivityOverCapCup*
\begin{proof} \hspace{1cm} \\
1.  $(V \cap V') \perp V''  = (V \perp V'') \cap (V' \perp V'')$: We show that if some $v \in (V \cap V') \perp V''$ then it must also be in $ (V \perp V'') \cap (V' \perp V'')$ and vice versa.
 \underline{Left-hand Side}: Let some $v \in (V \cap V') \perp V''$. Then it must be that $v \in (V \cap V')$ and ${\sim} v \in V''$. Since $v \in (V \cap V')$ and  ${\sim} v \in V''$ , it must be that $v \in (V \perp V'')$ and $v \in (V' \perp V'') $. Thus, $v \in  (V \perp V'') \cap (V' \perp V'')$.
 \underline{Right-hand Side}: Let some $v \in  (V \perp V'') \cap (V' \perp V'')$. Then $v \in (V \perp V'')$ and $v \in  (V' \perp V'')$ or $v \in V$, $v \in V'$ and ${\sim} v \in V''$. Thus $v \in (V \cap V') $ and therefore $v \in (V \cap V') \perp V''.$ $\Box $ \\ \\
2. $(V \cup V') \perp V''  = (V \perp V'') \cup (V' \perp V'')$: We show that if some $v \in (V \cup V') \perp V''$ then it must also be in $(V \perp V'') \cup (V' \perp V'')$. \underline{Left-hand Side}: Let some $v \in (V \cup V') \perp V''$.  Then $v \in (V \cup V')$ and ${\sim} v \in V''$. There are three cases to consider. a) $v \in V, v \not \in V'$: Thus $v \in  (V \perp V'')$ and therefore,  $v \in (V \perp V'') \cup (V' \perp V'')$.  b) $v \not \in V, v \in V'$: Similar to previous case. c) $v \in V, v \in V'$: Similar to previous case. \underline{Right-hand Side}: Let $v \in  (V \perp V'') \cup (V' \perp V'')$. There are three cases to consider. a) $v \in  (V \perp V''), v \not \in  (V' \perp V''), {\sim} v \in V'' $: Then $v \in V$ which means $v \in (V \cup V')$  and therefore $v \in (V \cup V') \perp V''$. b) $v \not \in  (V \perp V''),  v \in  (V' \perp V''), {\sim} v \in V''$: Similar to previous case. c) $v \in (V \perp V''), v \in (V' \perp V'') , {\sim} v \in V'' $: Similar to previous case. $\Box$
\end{proof}
\propBoldGreaterThanCautious*
\begin{proof}
The minuend in $Tr^K_B(C, a'')\allowbreak[bold] = \mid (V_A \cup V_B) \cap V_C \mid$ and the minuend in $Tr^K_B(C, a'')\allowbreak[cautious] = (V_A \cap V_B) \cap V_C$. Since $(V_A \cap V_B) \cap V_C \subseteq  (V_A \cup V_B) \cap V_C $, it follows $\mid (V_A \cap V_B) \cap V_C \mid \ \leq \ \mid (V_A \cup V_B) \cap V_C \mid$. The subtrahends $\mid (V_A \cup V_B) \perp V_C \mid$ are the same, so it must be that $Tr^K_B(C, a'')\allowbreak[bold] \geq Tr^K_B(C, a'')\allowbreak[cautious]$. $\Box$ 
\end{proof}
\propSemiIndependentGreaterThanCautious*
\begin{proof}
We know $Tr^K_B(C, a'')\allowbreak[semi\mh independent] =  \mid V_B \cap V_C \mid - \mid V_B \perp V_C \mid $. We know $Tr^K_B(C, a'')\allowbreak[cautious] = \mid (V_A \cap V_B) \cap V_C \mid - \mid (V_A \cup V_B) \perp V_C \mid$. Since $(V_A \cap V_B) \cap V_C \subseteq V_B \cap V_C $, it follows $ \mid (V_A \cap V_B) \cap V_C \mid \ \leq \ \mid V_B \cap V_C \mid$. Also, we know from Prop. \ref{propDistributivityOverCapCup} that $(V_A \cup V_B) \perp V_C = (V_A \perp V_C) \cup (V_B \perp V_C)$, so it follows that $\mid V_B \perp V_C \mid \ \leq \  \mid (V_A \cup V_B) \perp V_C \mid $. Thus,  $Tr^K_B(C, a'')\allowbreak[independent] \geq Tr^K_B(C, a'')\allowbreak[cautious]$. $\Box$
\end{proof}
\propCautiosLeqIndependent*
\begin{proof}
We already know from Prop.\ref{propBoldGreaterThanCautious} and Prop.\ref{propSemiIndependentGreaterThanCautious} that $Tr^K_B(C, a'')\allowbreak[cautious]$ is less that or equal to $Tr^K_B(C, a'')\allowbreak[bold]$ and $Tr^K_B(C, a'')\allowbreak[semi\mh independent]$.\ Since there are no conflicts of values between agents $Tr^K_B(C, a'')\allowbreak[independent] \ = \  \mid V_B \cap V_C \mid \ - \ 0 $ and $Tr^K_B(C, a'')\allowbreak[bold] \ = \ \mid (V_A \cup V_B) \cap V_C \mid \ - \  0$.  $V_B \cap V_C \subseteq (V_A \cup V_B) \cap V_C $, so $ \mid V_B \cap V_C \mid \ \leq \ \mid (V_A \cup V_B) \cap V_C \mid $ and thus $Tr^K_B(C, a')\allowbreak[independent] \leq Tr^K_B(C, a')\allowbreak[bold] $. $\Box$
\end{proof}
\theoremBoldGeqCautios*
\begin{proof}
Take any sequence $\mathcal{S}$ of length $n$ constructed using $Tr^K_{A_i}(A_{i+1}, a_i)\allowbreak[cautious]$. It is enough to show that we can construct a sequence $\mathcal{S'}$ using $Tr^K_B(\cdot)\allowbreak[bold]$ whose aggregate trust is greater or equal to that of $\mathcal{S}$.   
For $i = j = 1$, since we must use $Tr^K_{A_i}(A_{i+1}, a_i)\allowbreak[independent]$ for both $S$ and $S'$, $Q_1^1(S) = Q_1^1(S')$. 
When $i = 1$ and $j = 2$, for $\mathcal{S}$, let  $Tr^K_{A_2}(A_{3}, a_2)\allowbreak[cautious] = k$ and let $A_3$ be some agent  $X$. Now for $S'$, if there is an agent  $A_3 = Y$  such that  $Tr^K_{A_2}(Y, a_2)\allowbreak[bold] > Tr^K_{A_2}(X, a_2)\allowbreak[cautious] $, then $Q_1^2(S') > Q_1^2(S)$ as we previously established $Q_1^1(S) = Q_1^1(S')$. If there isn't one, for $S'$, we can still choose $X$ and we know from  Prop. \ref{propBoldGreaterThanCautious} that $Tr^{A_2}_B(X, a_2)\allowbreak[bold] \geq Tr^K_{A_2}(X, a_2)\allowbreak[cautious]$, thus  $Q_1^2(S') \geq Q_1^2(S)$. 
We can reason similarly for $2 < j \leq n$ and hence  $Q(\mathcal{S'}) \geq Q(\mathcal{S}) $. $\Box$
\end{proof}
%


\end{document}